\documentclass[journal]{IEEEtran}

\usepackage[utf8]{inputenc} 
\usepackage[T1]{fontenc}    
\usepackage{hyperref}       
\usepackage{url}            
\usepackage{booktabs}       
\usepackage{amsfonts}       
\usepackage{nicefrac}       
\usepackage{microtype}      
\usepackage{amsmath}		
\usepackage{amsthm}			
\usepackage{amssymb}		
\usepackage{multicol}
\usepackage{graphicx}	
\usepackage{subfig}
\usepackage{tikz}
\usepackage{cite}
\usepackage{mathtools}
\usepackage{bbm}
\usepackage{braket}
\usepackage{algorithm2e}
\usepackage[font=small,labelfont=bf]{caption}
\usepackage{pgfplots}
\usepgfplotslibrary{fillbetween}
\pgfplotsset{width=7cm,compat=newest}

\usetikzlibrary{bayesnet}
\usetikzlibrary{shapes,arrows,positioning,calc}
\usetikzlibrary{datavisualization}
\usetikzlibrary{datavisualization.formats.functions}
\usetikzlibrary{arrows}
\usetikzlibrary{positioning,fit}

\newtheorem{proto-theorem}{Proto-theorem}

\newtheorem{lemma}{Lemma}
\newtheorem{defn}{Definition}

\newcommand{\mc}[1]{\mathcal{#1}}

\DeclarePairedDelimiterX{\inp}[2]{\langle}{\rangle}{#1, #2}

\begin{document}

\title{Improving Supervised Phase Identification Through the Theory of Information Losses}

\author{Brandon~Foggo,~\IEEEmembership{Student Member,~IEEE} and~Nanpeng~Yu,~\IEEEmembership{Senior Member,~IEEE}
}
    
\tikzset{
		block/.style = {draw, fill=white, rectangle, minimum height=1em, minimum width=3em},
		tmp/.style  = {coordinate}, 
		sum/.style= {draw, fill=white, circle, node distance=1cm},
		input/.style = {coordinate},
		output/.style= {coordinate},
		pinstyle/.style = {pin edge={to-,thin,black}
		}
	}
     
\RestyleAlgo{boxruled}

\maketitle
    
\begin{abstract}
\footnote{© 20XX IEEE.  Personal use of this material is permitted.  Permission from IEEE must be obtained for all other uses, in any current or future media, including reprinting/republishing this material for advertising or promotional purposes, creating new collective works, for resale or redistribution to servers or lists, or reuse of any copyrighted component of this work in other works.} This paper considers the problem of Phase Identification in power distribution systems. In particular, it focuses on improving supervised learning accuracies by focusing on exploiting some of the problem's information theoretic properties. This focus, along with recent advances in Information Theoretic Machine Learning (ITML), helps us to create two new techniques. The first transforms a bound on information losses into a data selection technique. This is important because phase identification data labels are difficult to obtain in practice. The second interprets the properties of distribution systems in the terms of ITML. This allows us to obtain an improvement in the representation learned by any classifier applied to the problem. We tested these two techniques experimentally on real datasets and have found that they yield phenomenal performance in every case. In the most extreme case, they improve phase identification accuracy from $51.7\%$ to $97.3\%$.  Furthermore, since many problems share the physical properties of phase identification exploited in this paper, the techniques can be applied to a wide range of similar problems. 
\end{abstract}


\section{Introduction}
A power distribution circuit encompasses several components. It contains busses, powerlines, substations, regulators, transformers, and more. The physical arrangement of these components constitute the circuit's topology, which dictates much of the system's operation and planning. Over time, a circuit's topology will change. For example, a power outage may initiate a structural change such that the number of the effected customers is minimized. But the topology documentation will often not follow this change - they are only typically updated after major distribution expansion projects. As a result, there are long periods of time in which a circuit's topology is wrong, and this poses a serious problem. Power flow analysis, state estimation, and Volt-VAR control all depend on accurate topological information. When this is not available, the usefulness of those methods diminishes, and the system runs less effectively as a whole. Methods for faster documentation updates are necessary. 

Critical to the application of topology identification is the subproblem of phase identification. This subproblem describes the composition of each powerline in the network. Typically, a primary distribution powerline is made up of the four fundamental lines $A,B,C$ and $n$. The primary feeder, fed directly from the power substation, often consists of all four. However, at some point, a subset of these lines will be branched from the feeder. This change usually happens once along the path from the substation to any customer. As such, we define a customer's `phase type' as the lines branched along that path. This is illustrated in Figure \ref{fig:distr}. Knowledge of these phase types is necessary for estimation of electrical distances, and for topology reconstruction in general. In fact, since topology estimation is often taken as a precursor to state and parameter estimation in distribution networks, we can think of phase identification as the first step in a pipeline of distribution system modelling techniques.
\begin{figure}[t!]
    \centering
    \includegraphics[width=0.45\textwidth]{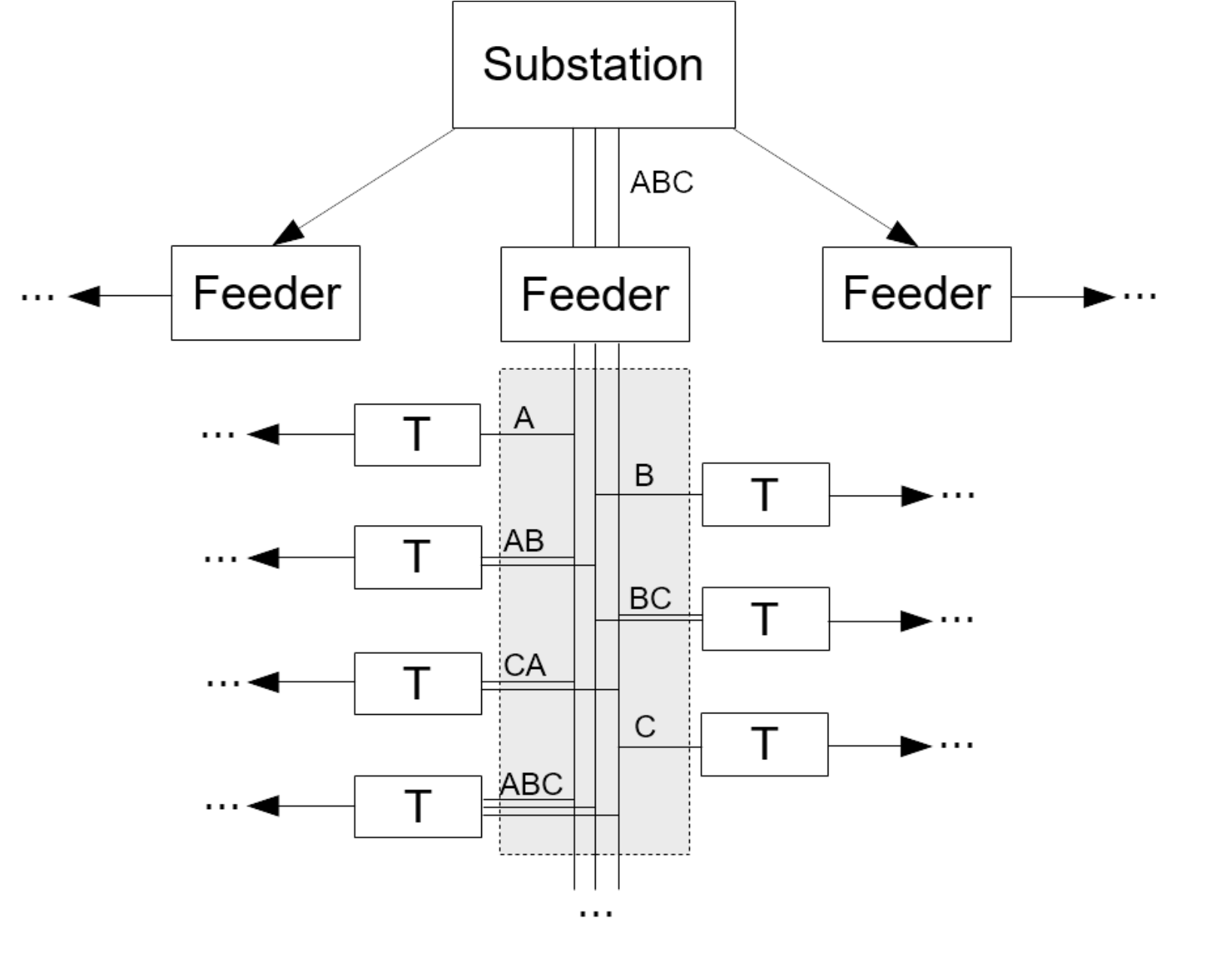}
    \caption{Illustration of the Phase Identification Problem. Here, each $T$ represents a transformer, and each combination of the letters $A$, $B$, and $C$ refers to a phase connection type.}
    \label{fig:distr}
\end{figure}

Literature on phase identification exists, but is limited. Like most power system applications, there are two broad classes of phase identification approaches - one model based, and the other data-driven. Like those other applications, model based methods have high interpretability but low accuracy, and data-driven methods have high accuracy but low interpretability. Our research poses itself as an intermediate between these extremes. This is done through Information Theoretic Machine Learning (ITML). This branch of machine learning connects traditional learning theory to the field of Information Theory. As such, interpretable notions such as entropy and mutual information are studied. These interpretable measures can be connected to the physics of the problems that we wish to solve.

We will focus on a particular regime of ITML which needs further development - the small data regime. Phase identification methods, particularly supervised methods, require a lot of training data to achieve high accuracies. But obtaining this training data requires portions of the lengthy field tests which we were trying to avoid in the first place. Thus we have a trade-off between phase identification accuracy and the amount of time that a distribution circuit's topology documentation is incorrect. In this paper, we will develop methods to obtain higher phase identification accuracies even when small training datasets are used. 

In particular, we will develop two techniques: Inverse Schur Training Data Selection, and Information Loading. The first corresponds to selecting the most informative data-points prior to field testing. The second consists of a modification to the objective function of a standard learning algorithm, with some modifications to the training phase in order to support it. We will heavily emphasize the theoretical reasoning behind these techniques, particularly in why they are applicable to the phase identification problem. Our proposed techniques have the following benefits: 
\begin{itemize}
	\item They require little infrastructure or physical labor. 
	\item They do not require any modeling of the network.
	\item They are robust with respect to missing data.
	\item They are easy to implement and tune.
	\item They can handle any variety of phase connections. 
	\item They provide the best supervised accuracies to date. 
	\item The representations learned are highly meaningful.
	\item They are generalizable to other networked systems.
\end{itemize}

\section{Related Work}
Much phase identification work is based on physical approaches \cite{chen2012, US8825416, wen2015, caird2012 }. References \cite{chen2012} and \cite{US8825416} develop phase identification systems based on high resolution timing measurements communicated between the base station and the feeder transformer secondaries and/or individual electricity consumers. This system is highly accurate and even yields the voltage phasors of the secondaries themselves instead of just the phase names of the wires connected to them. However, the system is quite sophisticated. Deploying such a system for each feeder across the many distribution circuits overseen by a utility is very expensive. 

Reference \cite{wen2015} describes a phase identification technique using micro-synchrophasors. The overhead cost of this method is less than the that of reference \cite{chen2012} since micro-synchrophasors are mobile - hence only a few devices are required. However, this reduced overhead cost results in increased labor and time costs, as the micro-synchrophasors must be reinstalled several times throughout the distribution network. Reference \cite{caird2012} patents a method for phase identification through signal injection. A signal generator is placed at the base substation and a unique signal is created for each phase. These signals are detected by a signal discriminator at each customer location. By matching the signals, the phase connectivities can be accurately reproduced. Like reference \cite{chen2012}, this method is very intensive in labor and time, but relatively cheap in overhead. These physical methods are the approaches that are typically taken during the field testing projects employed in practice. Furthermore, they act as the base upon which training labels should be acquired in application of this work. 	

The amount of literature related to solving of phase identification problem with data-driven methods is more limited. Of what does exist, most is unsupervised. This unsupervised branch can be split further into `model-agnostic' and `model-based' sub-branches. 

In the model-based sub-branch, reference \cite{985153} compares simulated power-flow solutions for a given phase configurations to real data. This method is accurate, but requires a correct system model including everything except phase connectivity's - e.g. line parameters, network topology. References \cite{7526150} and \cite{6102329} group customers by phase such that the total sum of power injections on each phase matches that of the substation or distribution transformers up to some error. These methods are somewhat non-robust to missing or erroneous data since these will lead to power mismatches between the measured load and the measured supply. Furthermore, phase configuration is often assumed to already exist in methods of network topology estimation and parameter estimation, so we typically like to think of phase identification as the first step in a long sequence of modelling techniques. Assuming realistic models as an input to the phase identification problem is somewhat unrealistic at this time.

In the model-agnostic sub-branch, one of the most popular techniques is to correlate voltage time series data at a household to the voltage time series data captured with other households \cite{liao2018unbalanced, 6465632, 6039623, 7604110, 7875102}. These works showcase the power of using voltage data as a primary predictor of phase type - an idea which we will study in depth in this paper. However, while the statement `customer A and customer B are on the same phase implies customer A and customer B have correlated voltages' is mostly accurate, these methods struggle in two ways. First, the converse of this statement is not always accurate. That is, customers can be on different phases and still have correlated voltages. Second, knowing that two customers are on the same phase doesn't actually tell us what that phase \textit{is}. This second issue can be resolved by either physical inspection, which are time and labor intensive, or by correlating the relevant voltages to those of the substation \cite{6365289}. But substation voltages are usually only measured either line to line or line to neutral - not both simultaneously, and so this limits the scope of circuits that we are capable of classifying.  

Finally, some unsupervised clustering techniques \cite{7838154, 8245748} and frequency domain feature extraction techniques \cite{7604110} have been tested on this problem. Supervised methods are limited to off-the-shelf algorithms that don't take any consideration into the properties of power distribution networks. To the author's knowledge, the method presented in this paper is the first of its type. 

\color{black}

\section{Background - Information Losses}
\subsection{Information Losses} \label{losses}
ITML studies the following loss term:
\begin{equation}\label{eqn:Iloss1}
    I_{Loss} = \left|I(Z,Y) - \mathbb{E}\left[log_2 \frac{\hat{p}(z,y)}{\hat{p}(z) \hat{p}(y)} \right]\right|
\end{equation}
This equation defines some notation which we will need to go over. First, $Y \in \mc{Y}$ is a random variable denoting the class variable in the machine learning problem. Here, $Y$ refers to the phase type of each customer. Next, we define a random variable $X\in \mc{X}$ (not appearing in (\ref{eqn:Iloss1})) used to represent the features of the machine learning problem. Here, $X$ will refer to a voltage time-series. Finally, $Z$ refers to some random variable generated as a stochastic function of $X$. This can technically refer to \textit{any} such random variable. However, we think of it as consisting of the internal state of a machine learning algorithm we wish to use.  For example, $Z$ may refer to the value of a hidden layer in a neural network. 

The term $I(Y;Z)$ refers to the mutual information between $Y$ and $Z$ \cite{cover2012elements}. 
\begin{equation}
	I(Y;Z) = \int log_2 \frac{dP_{YZ}}{d(P_Y \otimes P_Z)} dP_{YZ}
\end{equation}
The $\hat{p}$ notation refers to an estimated probability distribution - for example, a histogram in the discrete case, or a parametric model with a slightly incorrect parameter realization. We will further denote the rightmost term in equation (\ref{eqn:Iloss1}) as $\hat{I}(Y;Z)$. In general, all information theoretic terms pertaining to the estimated model will be denoted with a hat. 

We are interested, in particular, in these information quantities for two specific types of random variables:
\begin{defn}
Let $\epsilon>0$. Let $Z_{\epsilon}^*$ and $\hat{Z}_{\epsilon}$ be random variables that are at most ${\epsilon \text{-suboptimal}}$ for the following information optimization problems: 
\begin{align*}
\underset{p(z|x)}{sup}&{I(Y;Z)}  \\
\text{subject to }& I(X;Z) = I
\end{align*}
\begin{align*}
\underset{p(z|x)}{sup}&{\hat{I}(Y;Z)}  \\
\text{subject to }& I(X;Z) = I
\end{align*}
\end{defn}
We think of these representations as follows: $Z^*$ is the optimal estimator when we have perfect knowledge of our dataset, and $\hat{Z}$ is the optimal estimator when using a finite training dataset to estimate it. Both optimizations are subject to a constraint on the complexity of the estimator which comes from the Information Bottleneck method \cite{tishby2000information, hsu2006video, slonim2000agglomerative}.  
Given these definitions, we further define the information losses associated with the sample $S$ by:

\begin{equation}
I_{loss}(S) = |I(Y;Z^*) - I(Y; \hat{Z})|
\end{equation}
This quantity is strongly related to `minimal classification error'. There exist several inequalities indicating that classification error cannot be lower than a monotonic function of $I_{Loss}(S)$ \cite{shamir2008learning}. The most well known inequality of this type is due to Fano \cite{cover2012elements}, which state that, if $I(Y;Z^*) = H(Y)$ (i.e. the best \textit{possible} representation $Z^*$ is actually perfect), then:
\begin{equation}
	h_2(P_e) + P_elog_2\left(|\mathcal{Y}| - 1 \right) \geq I_{loss}(S) 
\end{equation}
where $P_e$ is the minimum probability of classification error for a classifier using representation $\hat{Z}$. 

It has been shown \cite{2019arXiv190205991F} that there is a relationship between $I_{loss}(S)$ and a measure known as the \textit{conditional total variation} which is defined bellow:
\begin{defn}
Given a conditional probability distribution $p(y|x)$ and an estimated distribution $\hat{p}(y|x)$, the conditional total variation of $\hat{p}(y|x)$ from $p(y|x)$ is given by 
\begin{equation}
\delta_{\hat{p}} = \mathbb{E}_{\mathbb{P}_X}\left[ \frac{1}{2}\sum_y \left|p(y|x) - \hat{p}(y|x)\right| \right]
\end{equation}
\end{defn}
This measure converges under standard machine learning regimes at least as fast as cross-entropy errors. The exact relationship between the term and information losses is given as follows:
\begin{lemma} \label{info_bound}
Let $\hat{p}(y|x)$ be a probability distribution estimate established via the the estimator $\hat{Z}_{\epsilon}(x)$. Let $h_2$ denote the binary entropy function ${h_2(t) = -tlog_2(t) - (1-t)log_2(1-t)}$. Then 
\begin{equation}\label{eqn:info_bound}
I_{loss}(S) \leq 2 \left( \vphantom{\sum} \delta_{\hat{p}}I(X;Z) + h_2(\delta_{\hat{p}}) \right) + \epsilon
\end{equation}
\end{lemma}

In this paper, we will consider techniques of mitigating information losses for the phase identification problem. This will be primarily through careful considerations of (\ref{eqn:info_bound}).

\section{Technique - Inverse Schur Training Data Selection}
\subsection{Motivation - Field Testing}
The first technique addresses the technical limitations of supervised phase identification. Namely, that obtaining labeled data is time consuming. Obtaining phase labels for a given customer requires on-site measurements with phasor measurement unit or phasing meter. Gathering phase connection information for a large number of customers with these equipment to serve as labels would be prohibitively costly.

It is critical that we get as much as we can out of just a few phase labels. This inspires the use of active learning \cite{settles2012active, tong2001active}. This is a field of machine learning which focuses on the selection of training data points that best represent the whole data set. It can be loosely divided into unsupervised methods and supervised methods - with the latter division dominating the number of available techniques by a large margin. Supervised methods generally update a machine learning model in conjunction with label selection. But this step is lengthy on its own, and can not be performed in parallel with on-location travelling or with device installation. Thus supervised techniques may be deemed too lengthy when data acquisition requires field testing. Thus we will focus on unsupervised techniques. We will build upon some recent theoretical work on the link between information losses and data selection \cite{foggo2019interpreting}.  

\subsection{Training Data Selection via Information Losses}
Recent theoretical work has found a bound on the term $\delta_{\hat{p}}$ in equation (\ref{eqn:info_bound}). It requires some notational setup to describe. 

We assume that we are given an unlabelled dataset $\mc{D}$ of size $N$. We wish to select a subset $S \subset \mc{D}$ of cardinality $M < N$. Let $k(\cdot, \cdot)$ be a continuous, symmetric, positive definite kernel function. Let $K$ denote the matrix

\begin{equation}
    \begin{bmatrix}
    k(x_1, x_1) & k(x_1, x_2) & \cdots k(x_1, x_N) \\
    k(x_2, x_1) & k(x_2, x_2) & \cdots k(x_2, x_N) \\
    \cdots & \cdots & \cdots \\
    k(x_N, x_1) & k(x_N, x_2) & \cdots k(x_N, x_N)
    \end{bmatrix}
\end{equation}
Let $S$ denote the selected training dataset, ${S \subseteq \mc{D}}$, and let ${SS =\{i: x_i \in S \}}$. Finally, let $K/K_{SS}$ denote the schur complement of $K$ with respect to $K_{SS}$. That is, if $\tilde{K}$ is the permuted form of $K$ such that the training indices correspond to the upper left block
\begin{equation}
    \tilde{K} = \begin{bmatrix}\tilde{K}_{SS} & \tilde{K}_{SU} \\ \tilde{K}_{US} & \tilde{K}_{UU} \end{bmatrix} 
\end{equation}
then
\begin{equation}
    K/K_{SS} = \tilde{K}_{UU} - \tilde{K}_{US}\tilde{K}_{SS}^{-1}\tilde{K}_{SU}
\end{equation}
Now let $T$ denote the following operator on $L_{\mathcal{P}_X}^2$:
\begin{equation}
    T[f](y) = \int k(y,x)f(x)d\mathbb{P}_X(x)
\end{equation}
The bound on $\delta_{\hat{p}}$ is expressed in the following lemma
\begin{lemma}
Assuming that our learner performs better than kernel optimization of the expected total variation, we have:
    \begin{align}
        \delta_{\hat{p}} \leq &\frac{p(Y=1)}{N}Trace(\sqrt{K/K_{SS}}) +\epsilon_{\mc{H}}
    \end{align}
    where $\epsilon_\mc{H}$ is given by
    \begin{align}
        \epsilon_\mc{H} = \frac{1}{2}\sum_{y}\|(I-T)p(Y=y|x)\|_{L_{\mathcal{P}_X}^1}
    \end{align}
$\sqrt{A}$ refers to the element-wise square root of the matrix $A$.
\end{lemma}

Thus to reduce information losses, we desire to minimize the term $Trace(\sqrt{K/K_{SS}})$ under a well fitting kernel (i.e. such that the second term is small). For the phase identification problem, and for problems whose datasets are generated by models with well-fitting linearizations, the cosine kernel works well. Thus we wish to solve
\begin{align}\label{prob:data_opt}
    \underset{S}{\text{min }} &Trace\sqrt{K/K_{SS}} \\
    s.t.~ &k(x,y) = \frac{x^Ty}{\|x\|\|y\|}
\end{align}
However, strict optimization of (\ref{prob:data_opt}) will take exponential time. Furthermore, since each evaluation of $K/K_{SS}$ for a given $S$ will require a matrix inverse calculation, greedy optimization performs poorly as well. Indeed, the computational complexity grows with $O(NM^4)$. If we choose $M$ such that $M=\rho N$ for some proportion $\rho$, then this scales with $O(N^5)$, which is quite poor. 

However, we can heuristically optimize (\ref{prob:data_opt}) in $O(N^3)$ time as follows:
First, we note that 
\begin{equation}
    K^{-1} = \begin{bmatrix} 
    \cdot & \cdot \\ \cdot &(K/K_{SS})^{-1}
        \end{bmatrix}
\end{equation}
Thus we can control the value of $Trace((K/K_{SS})^{-1})$ by data point selection in a predictable way. That is, the trace of $(K/K_{SS})^{-1}$ will be large if we pick points corresponding to small diagonal elements in $K^{-1}$.

We then note a few correspondences between the trace of a matrix and the trace of its inverse. First, by the Cauchy-Schwartz inequality, we have:
\begin{equation}
    len(A) \leq Tr(A)Tr(A^{-1})
\end{equation}
where $len(A)$ is the size of either axis of $A$. 

Furthermore, we have the following lemma:
\begin{lemma}
Consider the partially ordered  set of matrices $M = \{A: A=K/\{\alpha\} \text{ for some index $\alpha$ of size $M$}\}$ in the Loewner order. That is, $A \leq B$ iff. $A-B$ is positive semi-definite. Let $A^*$ be the matrix in this poset corresponding to the index $\alpha^*$ which maximizes $(K/\{\alpha^*\})^{-1}$. Then there exists no matrix $B \in M$ such that $A^* > B$. 
\end{lemma}
This lemma, which we will prove shortly, shows that picking data points to maximize $Tr((K/K_{SS})^{-1})$ will force $Tr(K/K_{SS})$ to be smaller than that of a large range of other datasets. Specifically, if $B$ is comparable to $A^*$ in the Loewner order (i.e. either $A^*-B$ or $B-A^*$ is positive semidefinite) then $A^*$ must be the smaller of the two and so $Tr(A^*) \leq Tr(B)$ by monotonicity of the trace operator.  We now proceed to this lemma's proof.

\begin{proof}
By the hypothesis of the lemma, $A^*$ has the largest inverse trace of all matrices in $M$. That is,  $Tr(C^{-1}) \leq Tr((A^*)^{-1})$ for all $C \in M$. Now suppose for the sake of contradiction that there exists $B \in M$ such that $A^* > B$. Then $(A^*)^{-1} < B^{-1}$ by monotanacity of the inverse operator. This, in turn, implies that $Trace((A^*)^{-1} ) < Trace(B^{-1})$, contradicting the hypothesis. 
\end{proof}

\section{Technique - Information Loading}
\subsection{Voltage Data and Phase Identification}
The second technique exploits the properties of our feature space. We will first describe what our feature space is and why we use it. We will then move into deriving the exploitable properties of this space - the main assertion being that a standardized voltage dataset has relatively low entropy. We will then experimentally validate this assertion. Finally, we will describe a technique, called information loading, which will exploit this property. 

Our feature space consists of smart-meter voltage magnitude time series data. Voltage data is fairly informative of phase type. Thus it is a good candidate feature space for the phase identification problem. This can be seen through the following example. 

Consider a power injection at bus $k$ whose phase type is $AB$. This induces a current along the lines $A$ and $B$. Thus, a voltage change will occur along those lines throughout the circuit. Any customer also feeding from either of those lines will notice a change. Due to the capacitive and inductive effects of the primary feeder, both lines will also induce a voltage change along the lines $C$ and $n$. However, the off-diagonal elements of the phase impedance and shunt admittance matrices are much smaller than the diagonal ones. Hence, the power injection at bus $k$ will have much less effect on phase $C$ than phase $A$ and $B$. Thus, the customers whose voltage data is most effected by this power injection are those on $AB$. Second to these customers are customers who share either the phase line $A$ or $B$ ($An, Bn, BC, CA$). Finally, customers who do not pull from either $A$ or $B$ ($Cn$) will hardly be effected at all.

Thus, while this current injection somewhat affects all customers, it affects customers of the same phase the most. Hence, voltage data is informative of phase connection type. 

\subsection{Properties of Voltage Data}\label{sub:properties}
We will now provide some rough quantitative descriptions of distribution systems. We consider a simple model in which each `customer' corresponds to a distribution transformer branched from the primary feeder. That is, we are ignoring the properties of the secondary distribution. We assume that bus $0$ represents the distribution substation and that each bus is a fixed electrical distance $\Delta_{E}$ from the previous bus. We will further assume that the impedance along the primary feeder is a constant $z$. We assume without loss of generality that the customers are indexed in order by their electrical distance from the distribution transformer. We let $V_{p}(x)$ denote the voltage on phase $p$ at position $x$. We assume that we are working in a per unit system and that our system is fairly balanced, such that ${V_A(0)\approx 1}$, ${V_B(0)\approx e^{-\frac{2\pi}{3}}}$, ${V_C(0)\approx e^{\frac{2\pi}{3}}}$, ${V_n(0)\approx 0}$. We will assume for simplicity that there are no three phase loads in the circuit. 

Now, let $\mu$ denote the degenerate measure $\sum_{j=1}^N \delta(j\Delta_{E})$ and let $i$ denote a stochastic process over $[0, N\Delta_{E}]$ taking complex values. When $i$ is evaluated at any $j\Delta_{E}, j=1,2,\cdots,N$, this is to be interpreted as the current injections into the system by customer $j$. Finally, we denote as $\vec{\phi}$ the function defining the phase of the current injection $i$. That is, if the current injection at position $x'$ is on phase $An$, then ${\phi(x') = \begin{bmatrix}1 &0 &0 &-1 \end{bmatrix}^T}$ whereas for phase $AB$ we would have ${\phi(x') = \begin{bmatrix}1 &-1 &0 &0 \end{bmatrix}^T}$. 

Now, denoting ${\vec{V}(x) \triangleq \begin{bmatrix}V_A(x) & V_B(x) & V_C(x) & V_n(x) \end{bmatrix}^T}$,  we have:
\begin{align}
	\vec{V}(x) \approx \vec{V}(0) - z\int_0^x x' i(x') \vec{\phi}(x') d\mu(x')
\end{align}     
Now, the voltage measured at $x$ by a meter, $\tilde{V}(x)$, is given by $\braket{\vec{\phi}(x),\vec{V}(x)}$. Thus:
\begin{align}\label{eqn:v_meas}
	& \tilde{V}(x=k\Delta_{E}) \notag \\
	\approx & \braket{\vec{\phi}(x),\vec{V}(0)}  - z\int_0^x  x' i(x')\braket{\vec{\phi}(x), \vec{\phi}(x')}d\mu(x') \\
	= & \braket{\vec{\phi}(x),\vec{V}(0)} \notag - z\Delta_{E}\sum_{j\leq k} j \cdot i(j\Delta_{E})\braket{\vec{\phi}(x), \vec{\phi}(j\Delta_{E})}
\end{align}
In which we already see our intuition pop out: the customers contributing most to the voltage measurement at $x$ are the customers whose $\vec{\phi}$ vectors have the largest inner product with $\vec{\phi}(x)$, i.e., the customers who share phase lines with the customer at $x$.

Now, (\ref{eqn:v_meas}) refers to phasor quantities. However, smart meters typically only return time averaged voltage magnitudes. We can take this into account by modifying each $\vec{\phi}$ vector in correspondence with the assumption that the circuit is fairly balanced. As such, we imagine lumping all of the single phase customers into one large wye-connected load and all of the two phase customers into one large delta-connected load. We consider each customer's current injections as contributing to current injections on these loads. Then the effect of all of the $A$ injections, for example, is to drop the voltage magnitude $V_A$ along the primary feeder, but not effect the neutral line at all. Thus, the $\vec{\phi}$ corresponding to phase A should be modified to ${\tilde{\phi} = \begin{bmatrix}1 & 0 & 0 & 0 \end{bmatrix}}$. Similarly, the $\vec{\phi}$ corresponding to phase AB should be modified to ${\tilde{\phi} = \begin{bmatrix}\frac{1}{\sqrt{3}} & -\frac{1}{\sqrt{3}} & 0 & 0 \end{bmatrix}}$ to account for the transformation from a line-line current magnitude into a line-neutral current magnitude. With all of the above approximations, and denoting as $\bar{V}$ the time-averaged version of $\tilde{V}$ we have
\begin{equation}
	\bar{V}(k\Delta_{E}) \approx g(k\Delta_E) -  |z|\Delta_{E}\sum_{j \leq k} j \cdot |i(j\Delta_{E})|c_{k,j}
\end{equation}
where $g(k\Delta_E)$ is $1$ if customer $k$ is single phase and $\sqrt{3}$ if customer $k$ is two phase. $c_{k,j}$ is an inner product with ${\frac{1}{\sqrt{3}} \leq |c_{k,j}| \leq \frac{2}{\sqrt{3}}}$. We can tighten the lower bound on the diagonal such coefficients as $1 \leq c_{i,i}$.

Collecting all $\bar{V}(k\Delta_{E})$ together into one vector $\bar{\mathbf{V}}$, we have:
\begin{equation}
	\bar{\mathbf{V}} = \mathbf{g} -  |z|\Delta_E\mathbf{C}\mathbf{J}\mathbf{I}
\end{equation}
where $\mathbf{g}$ collects the $g$ functions over the customers, $\mathbf{I}$ collects current magnitudes over each customer, ${\mathbf{J} = diag(1, 2, \cdots, N)}$, and 
\begin{equation}
\mathbf{C} = \begin{bmatrix}
c_{11} & 0 & 0 & \cdots & 0 \\
c_{21} & c_{22} & 0 & \cdots & 0 \\
c_{31} & c_{32} & c_{33} & \cdots & 0 \\
\cdots & \cdots & \cdots & \cdots & \cdots \\
c_{N1} & c_{N2} & c_{N3} & \cdots & c_{NN} \\
\end{bmatrix}
\end{equation}

Thus the statistics of $\bar{\mathbf{V}}$ are inherited linearly from those of $\mathbf{I}$. Assuming that $I$ is an iid Gaussian distribution with covariance $\sigma^2$, this yields a multivariate Gaussian distribution for $\bar{\mathbf{V}}$. The covariance of $\bar{\mathbf{V}}$ is given by ${\Sigma = |z|^2\Delta_{E}^2\sigma^2(\mathbf{CJ})(\mathbf{CJ})^T}$. We note that this iid Gaussian assumption is not prohibitive when our goal is in showing that our entropy is low; iid Gaussian random variables \textit{maximize} entropy at a fixed variance level. Thus, if the iid Gaussian assumption of $\mathbf{I}$ does not hold, then the entropy of the dataset will be even lower than what we predict. 

\subsection{Entropy Analysis}
The above approximate derivation allows us to analyze the entropic properties of smart-meter voltage measurements.

\begin{lemma}
For a distribution system following the model of subsection {\ref{sub:properties}}, we have:
\begin{align}
	ln\left(\frac{e}{N}\right) - \frac{1}{N}ln\left(\frac{\sqrt{\frac{2\pi}{e^2}}}{(N+1)}\right) \leq \frac{H(\bar{\mathbf{V}})}{N}  \notag \\
	\frac{H(\bar{\mathbf{V}})}{N} \leq \frac{1}{2} ln\left(\frac{12e}{N}\right) - \frac{1}{N}ln\left(\frac{\sqrt{\frac{e^2}{4\pi}}}{(N+1)}\right)
\end{align}
\end{lemma}

\begin{proof} We first note the growth rate of the diagonal elements of $\bar{\mathbf{V}}$'s covariance matrix $\Sigma$. 
\begin{align}
\Sigma_{11} \propto c_{11}^2 \\
\Sigma_{22} \propto c_{21}^2 + 4c_{22}^2 \\
\Sigma_{33} \propto c_{31}^2 + 4c_{32}^2 + 9c_{33}^2 
\end{align}
And thus we can estimate:
\begin{equation}\label{eqn:diag_cov}
\frac{1}{3}\frac{k(k+1)(2k+1)}{6} \leq \frac{\Sigma_{kk}}{\alpha} \leq \frac{4}{3}\frac{k(k+1)(2k+1)}{6}
\end{equation} 

where $\alpha = |z|^2\Delta_{E}^2\sigma^2$.

This estimate is important due to the preprocessing of our dataset. When applying a machine learning algorithm, we typically scale every data-set down such that the diagonal elements of the covariance matrix are equal to $1$. This makes the effective covariance ${\tilde{\Sigma} = D \Sigma D}$ for some matrix $D$ such that the diagonal of $\tilde{\Sigma}$ consists of $1's$. 

We can now estimate the entropy of a scaled voltage dataset via the entropy of a Gaussian random variable. 
\begin{align}
H(X) &= \frac{1}{2}ln\left((2\pi e)^N det(\tilde{\Sigma})\right) \\
& = \frac{1}{2}ln\left((2\pi e)^N det(\alpha DCJ^2C^TD)\right)
\end{align} 
Now, for calculating the determinant portion, we note that
\begin{equation}
det(\alpha DCJ^2C^TD) = \alpha^N det(D)^2det(C)^2det(J)^2
\end{equation}
First, we have $det(J)^2 = N!^2$. Further, we can bound  $det(C)$ as:
\begin{equation}
1 \leq det(C)^2 \leq \left(\frac{4}{3}\right)^N
\end{equation}
where the lower bound occurs if every customer is two-phase attached and the upper bound occurs if every customer is single-phase attached. Finally, due to (\ref{eqn:diag_cov}) and the scaling property:
\begin{equation}
\frac{9^N}{N!(2N+1)!(N+1)} \leq det(\alpha D)^2 \leq \frac{36^N}{N!(2N+1)!(N+1)} 
\end{equation}
Thus
\begin{align}
\frac{9^NN!}{\alpha^N(3N)!(N+1)} &\leq det(D)^2det(J)^2 \leq \frac{36^NN!}{\alpha^N(2N)!(N+1)}
\end{align}
Now, by Sterling's approximation, we can further obtain:
\begin{align}
\frac{\sqrt{\frac{2\pi}{e^2}}\left(\frac{e^2}{\alpha N^2}\right)^N}{(N+1)} \leq det(D)^2det(J)^2 \leq  \frac{\sqrt{\frac{e^2}{4\pi}}\left(\frac{9e}{\alpha N}\right)^N }{(N+1)}
\end{align}
From which we can finally obtain:
\begin{align}
	N ln\left(\frac{e}{N}\right) - ln\left(\frac{\sqrt{\frac{2\pi}{e^2}}}{(N+1)}\right) \leq H(\bar{\mathbf{V}})  \notag \\
	H(\bar{\mathbf{V}}) \leq \frac{N}{2} ln\left(\frac{12e}{N}\right) - ln\left(\frac{\sqrt{\frac{e^2}{4\pi}}}{(N+1)}\right)
\end{align}
\end{proof}
It is easy to see that both the lower bound and the upper bound become quite negative with $N$. For $N \approx 5000$, these bounds estimate the per-customer entropy to be roughly between $-4$ and $-11$ bits. If the smart meters are encoded with $16$ bits, then the per-customer entropy that the machine learning algorithm `sees' is between $5$ and $12$ bits. Thus we say that this problem is `inherently low entropy'. We note that the main reason this low entropy occurs is due to the fact that the covariance matrix diagonal values scaled with the distance of the corresponding customer from the substation. Had this not been the case - i.e., had every customer been of equal uncertainty, then the entropy of the dataset would have been much larger - scaling positively as $N$ grows large. 

We obtain a low entropy dataset because the uncertainty of the customer's voltage data is dominated by those customers far from the substation. This will stay true of many problems in networked systems. Since this is the main motivation behind the second technique introduced in this paper (information loading), we conjecture that the technique is highly applicable to problems beyond phase identification. 

\subsection{Maximum Mutual Information (MMI) Estimation: Further Evidence of the Low Entropy Feature Space Hypothesis} \label{entropy_est_sec}
For the next component of this paper, we will need a way of estimating the amount of mutual information that a given neural network \textit{can} carry about an input variable $X$. This estimation is the subject of this subsection.

We have from information theory that for any random variable $U$,
\begin{equation}
    H(X) = I(X;U) + H(X|U)
\end{equation}
We assume that $H(X)>0$ and suppose we have a class of conditional distributions $\mc{Q}$ from which to search for a variable $U$ meant to approximate $X$ (e.g. the hidden layer in our neural networks). Then if $\mc{Q}_{U|X}$ is large enough, it will contain a subset $\mc{Q}'$ such that $H(X|U') > 0$ for all $U'$ whose joint distribution follows $p(x) q'(u|x)$ for some  $q'(u|x) \in \mc{Q}'$. Then for such $U'$,  $H(X) \geq I(X;U)$, and we have equality when ${H(X|U) = 0}$. This motivates the estimator
\begin{equation}
    H(X) \geq \underset{q(u|x) \in \mc{Q}}{\text{sup }} I(X;U)
\end{equation}
whose optimum will not occur in $\mc{Q}'^{c}$. Thus we can then adapt the MINE-f mutual information estimator \cite{belghazi2018mine} to obtain
\begin{equation}\label{entropy_est}
    H(X) \geq \underset{q(u|x) \in \mc{Q}_{U|X}}{\text{sup }} \underset{t \in \mc{F}}{\text{sup }} \int t d\mathbb{P}_{XU} - \int e^{t-1} d\left(\mathbb{P}_X \otimes \mathbb{P}_U\right)
\end{equation}
where $\mc{F}$ is another space of functions. 
From here, we can simply define $\mc{Q}$ and $\mc{F}$ as spaces of functions parameterized by deep neural networks with fixed hyper-parameters. Since (\ref{entropy_est}) will only have equality for very large $\mc{Q}_{U|X}$, we will give the resulting right hand side a separate name - the \textit{maximum mutual information} (MMI) of $\mc{Q}_{U|X}$. This is a number which depends on the input feature space $\mc{X}$ and on the space $\mc{Q}_{U|X}$, and is the desired amount of mutual information that a given network (with fixed hyper-parameters) can carry about the input variable. 

We have experimentally conducted these MMI estimations on five real circuits with $\mc{F}$ defined by a feed-forward neural network with a single hidden layer of $1000$ units. We have observed a maximum MMI of just $8.67$ bits, and an average of just $6.21$ bits. For reference, these MMIs are much lower than that of the MNIST dataset, which has an MMI of about $21$ bits \cite{rippel2013high}. 

\subsection{Information Loading}
Let's quickly review what has been seen so far in this paper. In subsection \ref{losses}, we have the inequality ${I_{loss}(S) \leq 2 \delta(\hat{p})I(X;Z) + h_2(\delta(\hat{p}))}$. First, we used training data selection to reduce the term $\delta(\hat{p})$. We will now consider the other important term contributing to $I_{loss}(S)$:  $I(X;Z)$. At first glance, it would appear that \textit{reducing} $I(X;Z)$ would be a pertinent goal. Indeed this is the premise behind the information bottleneck method \cite{tishby2000information}. However, there is a hidden trade-off here. This is because reducing $I(X;Z)$ leads to its own form of information loss through the strong data processing inequality:
\begin{equation}
    I(Y;Z) \leq \eta I(X;Z), ~ \eta \leq 1  
\end{equation}
Thus we must balance the loss reduction in (\ref{info_bound}) against these additional losses. Furthermore, since we are already reducing $\delta(\hat{p})$ by using the data selection frameworks above, the marginal loss reduction that can be achieved from reducing $I(X,Z)$ is reduces as well. We thus conjecture that, unless $I(X;Z)$ is large, the losses from the data processing inequality will win out. But voltage data in general has relatively low entropy as we will show. Thus $I(X;Z)$ will be low as well for all random variables $Z$. Thus, we should attempt to keep $I(X;Z)$ as large as possible.  

However, there are losses in $I(X:Z)$ that occur naturally. First, if any stochasticity is introduced to the neural network in use, then we can view the neural network as a lossy channel. Information losses in $I(X;Z)$ will occur as a result. This can be alleviated by just not using a stochastic network, but there is a second form of $I(X;Z)$  losses that are more critical - finite data information losses. These losses come from the fact that, even though $X$ is instantiated for every customer, our neural classifier only sees $x$ instantiations that are accompanied by $y$ labels.  This artificially limits the amount of $X$ data that is seen, yielding the losses in $I(X;Z)$. 

But this need not be the case. If we can write an estimator of $I(X;Z)$ as a function of neural network parameters, then we can simply add in an information `anti-regularization' term to whatever loss function we are using. That is, if $\mc{L}$ is our current supervised loss function (say cross-entropy), then we can modify this to 
\begin{equation} \label{bottleneck_opp}
    \mc{L} - \beta I(X;Z),~ \beta>0
\end{equation}
effectively performing the \textit{opposite} of the information bottleneck method. 

We have such an estimator introduced in this paper already: MINE-f from subsection \ref{entropy_est_sec}. Plugging this estimator into (\ref{bottleneck_opp}) will yield the desired result. See figure \ref{fig:iloading:string}.

\begin{figure}[t]
\centering
\resizebox{0.45\textwidth}{4cm}{
\begin{tikzpicture}[thick]
  \node (blank) at (-2, 0)[draw=none, fill=none]{};
  \node (xj) at (0,0) [draw=none, fill=none]{$x_j$};
  \node (xm) at (0,-2) [draw=none, fill=none]{$x_m$};
  \node (xp) at (0,-3) [draw=none, fill=none]{$x_p$};
  \node (x) at (0,-4) [draw=none, fill=none]{$x_{tr}$};
  \node (y) at (0,-5) [draw=none, fill=none]{$y_{tr}$};
 
  \node (r1) at (2,0) [rectangle,draw=black] {$r(\cdot)$};
  \node (t1) at (4,0) [rectangle,draw=black] {$t(\cdot, \cdot)$};
  
  \node (p1) at (6, -1) [circle, draw=black] {$+$};
  \node (b) at (8, -1) [rectangle,draw=black] {$-\beta$};
  
  \node(p2) at (8, -2.5) [circle, draw=black] {$+$};
  
  \node (r2) at (2,-2) [rectangle,draw=black] {$r(\cdot)$};
  \node (t2) at (4,-2) [rectangle,draw=black] {$t(\cdot, \cdot)$};
  \node (e) at (6,-2) [rectangle,draw=black] {$-e^{(\cdot) - 1}$};
  
  \node (r3) at (2,-4) [rectangle,draw=black] {$r(\cdot)$};
  \node (s) at (4,-4) [rectangle,draw=black] {$\sigma(\mathbf{w}_{\theta_c}^T\cdot)$};
  \node (h) at (6, -4) [rectangle,draw=black] {$H_{p,q}(\cdot, \cdot)$};
  
  \node (l) at (10, -2.5) [draw=none, fill=none]{$\mc{L}$};

  \draw[] (xj) -- (r1);
  \draw[o-] (r1) -- (t1);
  \draw[] (xj) -- ++(1, -1) -| (t1.south);
  
  \draw[] (xm) -- (r2);
  \draw[o-] (r2) -- (t2);
  \draw[] (xp) -- ++(4, 0) -| (t2.south);
  \draw[] (t2) -- (e);
 
 \draw[] (e) -- (p1);
 \draw[] (t1) -- ++ (2, 0) -| (p1);
 
  \draw[](x) -- (r3);
  \draw[](r3) -- (s);
  \draw[](s) -- (h);
  \draw[](y) -- ++(6,0) -| (h);
  
  \draw[] (p1) -- (b);
  \draw[] (b) -- (p2);
  \draw[] (h) -- ++(2, 0) -| (p2);
  
  \draw[] (p2) -- (l);
  
\end{tikzpicture}
}
    \caption{String diagram representation of the information loading forward pass. $r: \mc{X} \to \mc{Z}$ is the representation function. ${t: \mc{X} \times \mc{Z} \to \mathbb{R}}$ is the argument of the information estimator.}
    \label{fig:iloading:string}

\end{figure}
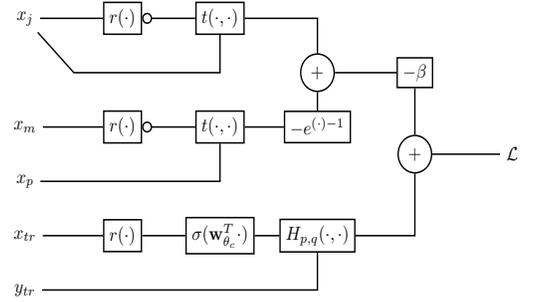

\section{Experiments}   

\subsection{Data}
This analysis will be performed over $5$ circuits of varying complexity from Southern California Edison, Pacific Gas and Electric Company, and FortisBC. The details of these circuits are contained in the Table \ref{CharTable}.
	\begin{table}[h]
	\centering
	\caption{Distribution Circuits Characteristics}
		\begin{tabular}{| c | c | c | c |}
			\hline
			Name & $N_{consumers}$ & Phase Connections & Degree of Balance \\ \hline
			I  & 1892 & $A,B,C,ABC$ & Low \\ \hline
			II & 3166 & All & Low \\ \hline
			III & 4629 & $A,B,C,ABC$ & High \\ \hline
			IV & 3638 & $B,C,AB,BC,CA$ & High \\ \hline
			V   & 1773 & $A,B,C$ & Low \\ \hline
		\end{tabular}
	\label{CharTable}
	\end{table}
	Where Degree of Balance is measured by the average current coming back on the neutral line in the distribution circuit. The obtained values are partitioned into two equal-probability bins which we denote as `Low' and `High'. \color{black}
	
	Each circuit contains $31$ days of voltage magnitude data, sampled hourly for a feature vector of dimension $744$. All experiments are performed with $5\%$ of the total customers used as training data. Every reported number is an average over ten trials. 
		
	Empirically, circuits with more potential phase connections (e.g. $A$, $B$, $C$, $AB$, $BC$, $CA$ vs. just $A$, $B$ and $C$) typically have lower Phase Identification accuracy. This is firstly due to the fact that the difficulty of a classification task is related to the number of classes, but also due to the fact that there are nontrivial dependencies between some of these classes; for example, transformers of the $AB$ class take current from the $A$ line and send it back along the $B$ line, which complicates the dynamics of transformers attached to just $A$ or $B$. Balanced circuits also have lower Phase Identification accuracy than unbalanced ones, but the effect is less significant. The more phase connections available and the more balanced the circuit is, the more 'difficult' that circuit is to identify.

\subsection{Preprocessing}
In many distribution systems, center tapped transformers are abundant. As such, voltage magnitude data will come in two bulk clusters, one near $120V$ and one near $240V$. This distinction has little relevance for the phase connection type of the corresponding customer, and will add instability into any supervised learning algorithm. We take care of this by \textit{self}-normalizing each voltage time series such that its time-average is equal to $1.0$.  We follow this step with the standard preprocessing technique of \textit{batch}-normalizing the data to have an batch-mean of $0.0$ and a batch-standard deviation of $1.0$. 

\subsection{Results}
	We first desired to establish some baseline accuracies for the phase identification problem using standard supervised learning approaches. The results of this analysis are shown in Table \ref{tab:baseline_experiments}. We see that a two layer neural network with $500$ hidden units outperforms the other methods in $4$ out of $5$ cases. Only on circuit II is the neural network beaten, and barely so. Thus we decided to implement our changes on this classifier. 
	
	We tested our proposed techniques in all permutations. These are all shown in Table \ref{tab:experiments}. The first column of this table repeats the accuracy of the baseline two layer neural network from figure \ref{tab:baseline_experiments}. The second column considers information loading in isolation. This yields minor to substantial changes. The effect is highly dependent on the circuit. In general, this technique seems to yield larger improvements to harder circuits. We then tested, training data selection under the facility location method \cite{sener2018active}. This yields substantial improvement in every case. We observe in the fourth  column that the inverse-matrix heuristic outperformed the facility location method in every case. Finally, we performed both the information loading technique and the inverse-schur heuristic to obtain the rightmost column in the table. These combined techniques yield the highest phase identification accuracies in every case. They are phenomenally improved from the baseline results. 

\begin{table}[h]
    \centering
    \begin{tabular}{@{}l l l l l l l@{}}\toprule
         Circuit  & Neighbors & Decision Tree & Random Forest & Neural (2-layer) \\ \hline
         \rule{0pt}{2.5ex}  I &  74.6\% &  71.5\% & 68.2\% & \textbf{80.7}\% \\ 
         \rule{0pt}{2.5ex}  II &  \textbf{64.8\%} &  59.3\% & 39.3\% & 64.7\% \\ 
         \rule{0pt}{2.5ex}  III &  70.6\% &  59.2\% & 59.4\% & \textbf{71.4\%} \\ 
         \rule{0pt}{2.5ex}  IV  &  67.2\% &  59.3\% & 51.8\% & \textbf{75.0\%} \\ 
         \rule{0pt}{2.5ex}  V   &  41.0\% & 50.00\% &  37.00\% & \textbf{51.7\%} \\ \hline
    \end{tabular}
    \caption{Baseline Establishment}
    \label{tab:baseline_experiments}
\end{table}

\begin{table}[h]
    \centering
    \begin{tabular}{@{}l l l l l l l@{}}\toprule
         Circuit  & Baseline & I-loading & Facility & Inverse & Inverse+I-loading   \\ \hline
         \rule{0pt}{2.5ex}  I & 80.7\% & 81.5\% & 86.7\% & 89.9\% & \textbf{91.0\%} \\ 
         \rule{0pt}{2.5ex}  II & 64.7\% & 80.6\% & 90.5\% & 95.4\% & \textbf{96.3\%} \\ 
         \rule{0pt}{2.5ex}  III & 74.1\% & 75.2\% & 90.6\% & 91.5\% & \textbf{93.1\%} \\ 
         \rule{0pt}{2.5ex}  IV  & 75.0\% & 78.0\% & 91.2\% & 94.6\% & \textbf{98.8\%} \\ 
         \rule{0pt}{2.5ex}  V   & 51.7\% & 59.1\% & 94.2\% & 96.1\% & \textbf{97.3\%} \\ \hline
    \end{tabular}
    \caption{Proposed Techniques}
    \label{tab:experiments}
\end{table}

\begin{table}[h]
    \centering
    {\begin{tabular}{@{}l l l l l l l@{}}\toprule
         Circuit  & Correlation & Clustering & Proposed   \\ \hline
         \rule{0pt}{2.5ex}  I & 37.8\%  &75.1\%   & \textbf{91.0\%} \\ 
         \rule{0pt}{2.5ex}  II & 34.1\%  &56.4\%  & \textbf{96.3\%} \\ 
         \rule{0pt}{2.5ex}  III & 46.4 \%  &65.7\%  & \textbf{93.1\%} \\ 
         \rule{0pt}{2.5ex}  IV  &  40.1\%  &53.6\% & \textbf{98.8\%} \\ 
         \rule{0pt}{2.5ex}  V   & 38.4\%  &38.4\%  & \textbf{97.3\%} \\ \hline
    \end{tabular}}
    \caption{Accuracy comparisons between the literature and the proposed method.}
    \label{tab:experiments_lit}
\end{table}

\begin{figure}[t!]
    \centering
    \hspace{25em} \includegraphics[width=0.45\textwidth]{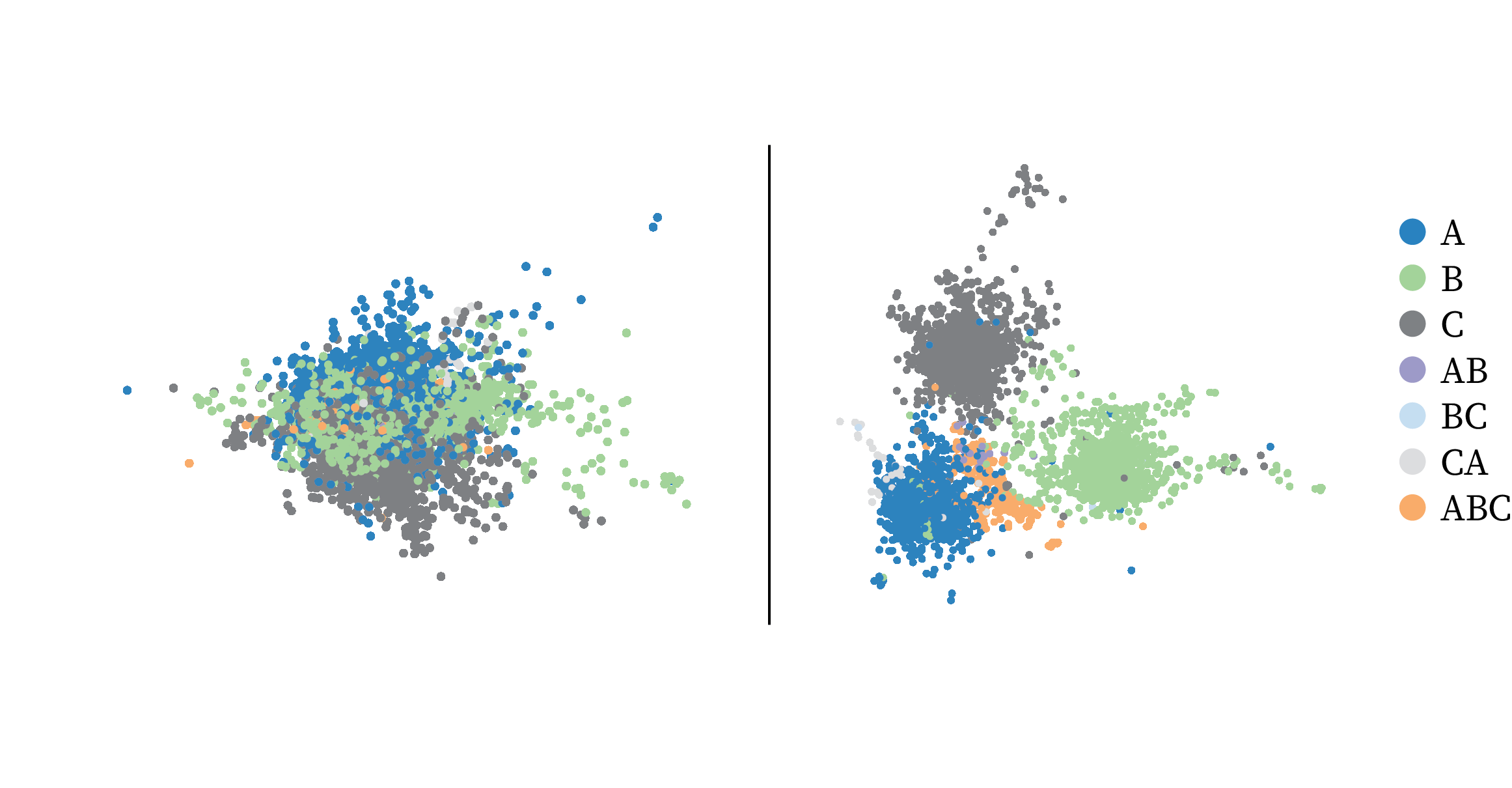}
    \caption{Learned representations on circuit II. (Left) random selection with no information loading. (Right) Targeted selection with information loading.}
    \label{fig:representations}
\end{figure}

\begin{figure}[t!]
    \centering
    \hspace{25em} \includegraphics[width=0.45\textwidth]{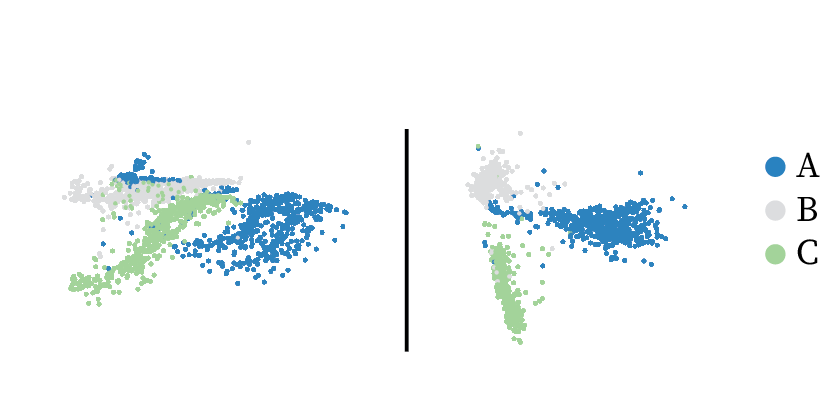}
    \caption{Learned representations on circuit V. (Left) random selection with no information loading. (Right) Targeted selection with information loading.}
    \label{fig:representations_s}
\end{figure}

\begin{figure}[t!]
    \centering
    \includegraphics[width=0.5\textwidth]{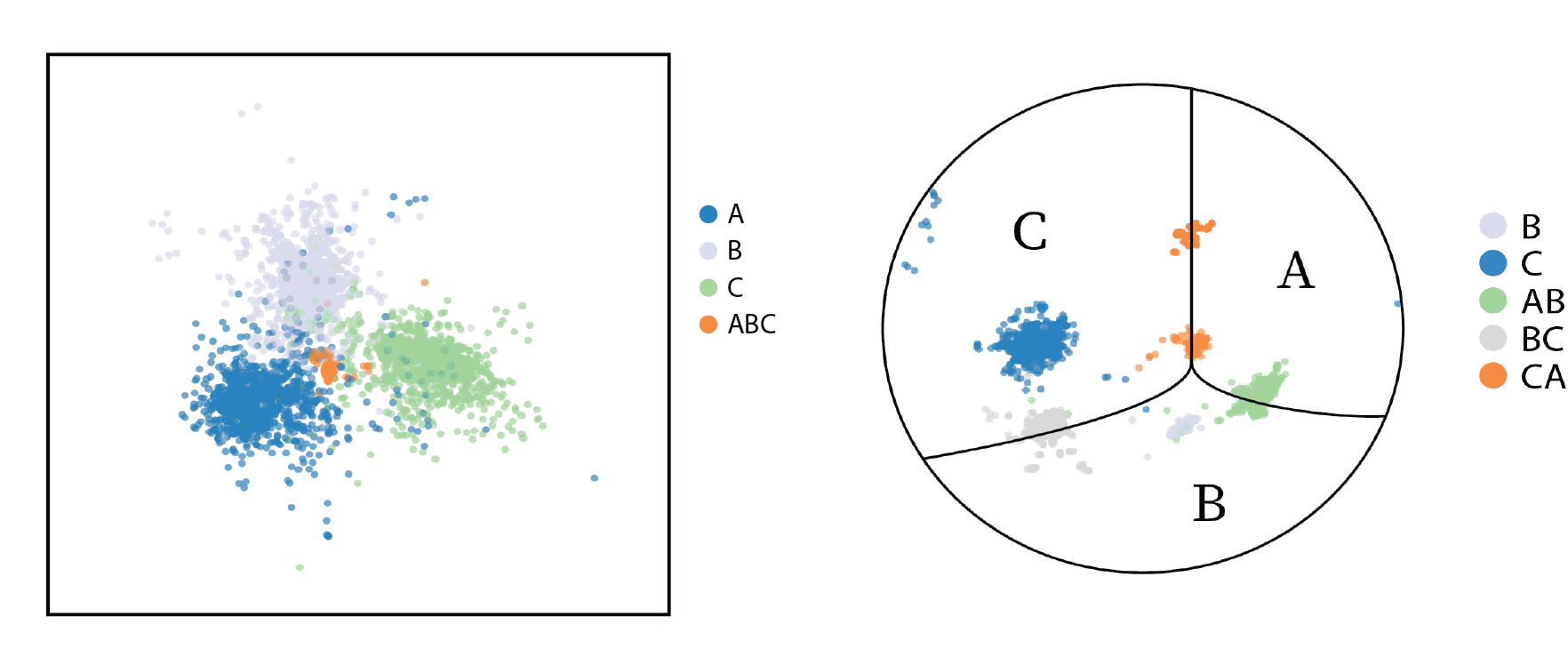}
    \caption{Learned representations on circuits I (left) and IV (right) with both techniques implemented.}
    \label{fig:representations_102_BH}
\end{figure}

We've also compared the accuracies obtained by our proposed method to two of the methods of phase identification in literature. Both of these methods lie on the physical-intuition side of techniques, in contrast to the more abstract off-the-shelf machine learning techniques. The first, which we've denoted `correlation', slightly modifies the correlation based methods of references \cite{liao2018unbalanced, 6465632, 6039623, 7604110, 7875102} by computing the empirical voltage correlation matrix over the customers, and using this matrix to link the customers together under complete-linkage. The second, which we've denoted `clustering' is equivalent to that of reference \cite{7838154}. We see that the literature is capable of achieving similar accuracies to the more abstract machine learning algorithms in Table \ref{tab:baseline_experiments}, but without requiring training data. However, our proposed method, which synergistically combines both physical intuition and the more abstract learning algorithms, beats both types of methods individually, yielding the best results on these datasets by a large margin.

Since the central idea of this paper lies in learning better representations, we ought to visualize what the representations have learned. This has been done for two of the circuits, and is visualized in figures \ref{fig:representations} and \ref{fig:representations_s}. Each point on these plots corresponds to an averaged sample from the distribution $p(z_i|x_i)$ averaged over $50$ trials. Each of these averages have dimension $500$. We reduced this to $2$ dimensions by performing PCA and projecting each point onto the first two principal components.

We observe that the representation has indeed learned much better representations by using these techniques. In the baseline case, there is little separation between the classes - especially in the case of circuit II. By implementing the proposed techniques, we see more class separation in these variables. Furthermore, we observe some learning of the relationships between the classes. For example, in circuit II, the learned representation has placed $ABC$ directly in the middle of the classes $A$ $B$ and $C$. This is significant since, by subsection \ref{sub:properties}, phase $ABC$ truly is a `combination' of phases $A$, $B$, and $C$ when it comes to voltage data. In the case of circuit V, we have learned even more. In the final representation, we see three branches of representation data - one for each class. The distance from the central node of these representations corresponds to distance along the primary feeder of the distribution circuit. 

Finally, we have plotted the final representations under both techniques for circuits I, and IV in figure \ref{fig:representations_102_BH}. The behavior of circuit III is nearly identical to circuit I, so we are not presenting its plot for the sake of space. In the cases of circuits I and III, we see similar behavior. Independent $An$ $Bn$ and $Cn$ clusters have arisen, and an $ABC$ cluster has appeared in the middle. This is, again, what we expect from circuits consisting of those phase types. Circuit IV is a bit more interesting because it consists of phase types $Bn, Cn, AB, BC$ and $CA$, with $Bn$ having very little representation. Clusters of each phase type have appeared with decent separation, so classification will at least be easy. More interestingly, however, is the fact that the $AB$ cluster has appeared opposite to the $Cn$ cluster, $CA$ opposite $Bn$. Assuming the location of the non-existent $An$ cluster to the top right, these positions make much sense.

\section{Conclusion}
 This paper has used the theory of information losses to propose the application of two novel techniques - inverse schur data selection and information loading - to the phase identification problem. These techniques have synergistically combined the the abstract, problem agnostic, supervised learning techniques found in machine learning research with the physical intuitions that are often employed in more specific power systems projects. As such, we observe substantial improvements in phase identification accuracy over both the purely abstract methods and those methods which only base themselves on the physical intuitions. Furthermore, we have observed that the representations learned upon using these techniques are much more meaningful than the baseline representations,  giving us highly interpret-able results which are often not found in techniques which use abstract machine learning algorithms alone. We have argued that these techniques generalize well beyond just phase identification, and have listed properties for which these techniques will be helpful.  

\bibliographystyle{ieeetran}
\bibliography{references}

\end{document}